\documentclass{article}

\usepackage[preprint]{neurips_2025}
\usepackage{amsmath}
\usepackage{amssymb}
\usepackage{mathtools}
\usepackage{amsthm}
\usepackage{hyperref}       
\usepackage[capitalize,noabbrev]{cleveref}
\usepackage{bbm}

\theoremstyle{plain}
\newtheorem{theorem}{Theorem}[section]

\theoremstyle{definition}

\theoremstyle{remark}

\newcommand{\argmax}{\mathop{\mathrm{argmax}}\limits}
\newcommand{\argmin}{\mathop{\mathrm{argmin}}\limits}

\usepackage[utf8]{inputenc} 
\usepackage[T1]{fontenc}    
\usepackage{url}            
\usepackage{booktabs}       
\usepackage{amsfonts}       
\usepackage{nicefrac}       
\usepackage{microtype}      
\usepackage{xcolor}         

\title{Beyond Ordinal Preferences: Why Alignment Needs Cardinal Human Feedback}

\author{%
  Parker Whitfill\\
  Department of Economics\\
  Massachusetts Institute of Technology\\
  Cambridge, MA\\
  \texttt{whitfill@mit.edu}
  \And
  Stewy Slocum\\
  Computer Science and Artificial Intelligence Laboratory (CSAIL)\\
  Massachusetts Institute of Technology\\
  Cambridge, MA\\
  \texttt{sslocum@csail.mit.edu}
}

\begin{document}

\maketitle

\begin{abstract}

Alignment techniques for LLMs rely on optimizing preference-based objectives -- where these preferences are typically elicited as ordinal, binary choices between responses. Recent work has focused on improving label quality or mitigating particular biases, but we identify a more fundamental limitation: these methods collect the wrong kind of data. We prove an impossibility result: no algorithm relying solely on ordinal comparisons can systematically recover the most preferred model. Intuitively, ordinal data lacks the information needed to resolve tradeoffs -- e.g., fixing a factual error on one prompt versus improving style on another. We show that selecting the optimal model requires recovering preferences over \emph{models} (rather than just responses), which can only be identified given cardinal feedback about response quality. To address this, we collect and publicly release a dataset of 25,000 cardinal judgments using willingness-to-pay elicitations, a well-established tool from experimental economics. Empirically, we find that incorporating cardinal feedback into preference fine-tuning allows models to prioritize high-impact improvements and outperform ordinal-only methods on downstream benchmarks, such as Arena-Hard.

%
\end{abstract}

\section{Introduction}

Optimizing preference-based metrics has become an important part of the fine-tuning process for LLMs. However, recent studies demonstrate that these metrics often reward superficial tricks such as replying at greater length, polishing style, or being sycophantic, sometimes without addressing factual errors or safety concerns \citep{singhal2023long, li2024style, pasch2025llmcontentmoderationuser, singh2025leaderboardillusion}. While previous work has focused on data quality \citep{zhao2024challengestrustworthyhumanevaluation}, temporal overfitting \citep{shirali2023theorydynamicbenchmarks}, and mitigating specific superficial biases \citep{park2024disentanglinglengthqualitydirect}, we identify a novel limitation of existing systems -- that they collect the wrong kind of preference data. When feedback is ordinal -- ``response A is better than response B", we show there is insufficient information to identify the most preferred model. An algorithm has no way to tell whether fixing a critical safety flaw on one prompt matters more than polishing wording on another, because both improvements are logged as a single ``win''.

Our main theoretical result formalizes this intuition: with access to only ordinal comparisons, no fine-tuning procedure can systematically output the most preferable, feasible model. Figure \ref{fig:headline} illustrates the idea: a model that eliminates a hazardous hallucination in a medical recommendation versus one with improved spelling and stylistic flair. With only ordinal preference data, the algorithm does not know to prioritize avoiding the medical error. Instead, we propose gathering cardinal feedback directly from humans, letting them assign higher importance to the safety fix.

Previous alignment work has avoided cardinal human feedback for fear of noise, calibration issues, and cognitive load \citep{christiano2017deep, casper2023open}. We tackle these issues with a simple, intuitive protocol borrowed from experimental economics: ask annotators for their willingness-to-pay (WTP) for a proposed improvement to an LLM completion. WTP questions have been validated in hundreds of lab and field studies as a robust way to elicit the strength of preferences without demanding technical expertise \citep{carson2005contingent, Schmidt2020HypotheticalBiasWTP, Cerda2024RenewablesWTP}. 

Using conversations from ChatbotArena \citep{lmsys-chatbot-arena} supplemented with Anthropic's HH dataset \citep{bai2022training}, we collect over $25,000$ human WTP judgments on LLM completions, forming the \textsc{CardinalPrefs} dataset\footnote{\textsc{CardinalPrefs} dataset available at \url{https://huggingface.co/datasets/cardinal-prefs/CardinalPrefs}}, which we publicly release to support future research. We find that incorporating this data into preference-tuning allows training to prioritize high-impact improvements and yields models that are rated as more preferable and score higher on downstream preference-based benchmarks. 

In summary, we provide the following contributions:
\begin{enumerate}
    \item We prove that any alignment algorithm that only utilizes ordinal feedback cannot systematically select the most preferable model. Intuitively, not accounting for preference strength makes it impossible to correctly resolve tradeoffs across prompts. 
        
    \item We show that existing ordinal feedback models like Bradley-Terry make implicit assumptions about an unobserved cardinal feedback signal, but that these assumptions do not match real-world data, and that explicit cardinal feedback measures this better. 
    
    \item We empirically demonstrate that training on cardinal feedback leads to stronger alignment with human preferences and consistent improvements on various preference-based metrics. We show these gains arise from prioritizing high-impact fixes, while ordinal preference training gives every comparison equal weight.
\end{enumerate}


\begin{figure}
    \centering
    \includegraphics[width=\linewidth]{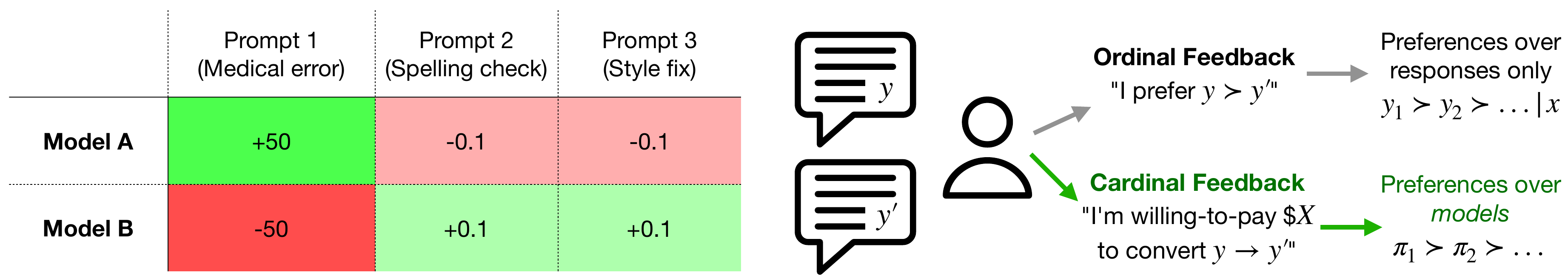}
    \caption{\textbf{Left: Ordinal feedback can select the wrong model.} Ordinal wins favor Model B on 2/3 prompts, yet cardinal feedback shows Model A delivers far higher overall utility by fixing a critical medical error. Intuitively, cardinal feedback allows for making correct tradeoffs across prompts. \textbf{Right: Cardinal feedback and preferences over models.} We elicit cardinal feedback through willingness-to-pay estimates out of a total fixed budget across all prompts. This yields a numeric utility difference instead of binary win/lose score. This additional information above ordinal data is necessary to identify model-level preferences. 
    }
    \label{fig:headline}
\end{figure}

\section{Related Work}\label{sec:related_work}

\textbf{Ordinal preference learning algorithms.}
Recent work has introduced a plethora of preference learning algorithms, which consume ordinal, binary feedback (e.g., RLHF \citep{christiano2017deep}, DPO \citep{rafailov2024direct}, and IPO \citep{azar2024general}), binary ratings (KTO \citep{ethayarajh2024kto}), linguistic principles (Constitutional AI \citep{bai2022constitutionalaiharmlessnessai}), and others. Most related to our work, recent papers such as \citep{gao2024rebel, amini2024direct} have discussed modifications to incorporate intensity of preference. However, while both papers try to adjust the algorithm to take advantage of differences in intensity of preference, they still get the intensity of preference from a reward model learned from ordinal feedback. Our novel contribution to this literature is changing the data collected. 

\textbf{Learning from scalar feedback.} Recognizing the greater expressivity of cardinal over ordinal feedback, previous work has used scalar human feedback to train reinforcement learning models \citep{wilde2021learningrewardfunctionsscale, knox2008tamer, macglashan2023interactivelearningpolicydependenthuman}. However, it is often not clear to human annotators how to quantify the success of an example, even after discretization (e.g., into a Likert scale) \citep{yannakakis2011ranking, christiano2017deep}. We contribute to this literature by introducing a simple way to elicit cardinal feedback via WTP in a modern, LLM setting.  

\textbf{Economic decision theory.} 

Our framework builds on classical economic decision theory \citep{debreu1960topological, savage1972foundations, anscombe1963definition, fishburn1970utility} - a literature that focuses on properties of preference relations, but we bring this theory to a new setting of LLMs. 

%

\section{Preliminaries}

Let $\mathcal{X}$ be a finite prompt space and $\mathcal{Y}$ be a finite response space. A language model $\pi: \mathcal{X} \to \Delta(\mathcal{Y})$ associates each prompt with a discrete distribution over responses, and we denote the set of all such models as $\Pi = \{\pi \mid \pi: \mathcal{X} \to \Delta(\mathcal{Y})\}$.

\subsection{Preferences}

Let $P(\cdot)$ be the set of complete, transitive, continuous and independent preference relations over a given set. We review these axioms in the appendix. 

Let $\succsim \in P(\Pi)$ be a preference relation defined over $\Pi$, the set of language models. This preference relation encodes a user's preferences over \emph{models} (e.g., do they prefer GPT-$4$o or Grok $3$). Our preferred interpretation is that $\succsim$ is the preference of the model developer.\footnote{Therefore, we are not facing a social choice problem of aggregating preferences across people, but rather trying to optimize a single preference, the developer's. }

In preference fine-tuning, it is common to elicit preferences over \textit{responses} for a fixed prompt. Let $\succsim_x \in P(\Delta(\mathcal{Y}))$ be a preference relation over \textit{responses} for a given prompt $x \in \mathcal{X}$.\footnote{We will abuse notation and say $y \succsim_x y'$ to refer to a lottery which puts probability $1$ on $y$ is preferred to a lottery which puts probability $1$ on $y'$.} We call these \emph{prompt-level preferences}. Prompt-level preferences are derived from model-level preferences by comparing two models that differ only with respect to a single prompt. See the appendix for details.

\subsection{Preference Fine-Tuning Algorithm}

Let $D_\succsim$ denote a preference dataset. Let $f$ be a preference fine-tuning algorithm that selects a language model from a set of feasible models $\mathcal{F} \subseteq \Pi$ based on the dataset $D_\succsim$. 
\[f(D_\succsim, \mathcal{F}) = \pi \quad \text{ where } \pi \in \mathcal{F}\]

We interpret $\mathcal{F} \subseteq \Pi$ as the set of models that are sufficiently close to the pretrained model, and that are representable given our neural network weights.\footnote{In most applications, KL regularization is applied as an explicit penalty; by Lagrange duality, this is equivalent to restricting $\mathcal{F}$ to policies within an $\varepsilon$ ball around the pretrained model. For mathematical simplicity, we will assume $\mathcal{F}$ is compact throughout the paper. }

Let us review RLHF as an example choice of $f$
\[f_{RLHF}(D_\succsim^\text{ordinal}, \mathcal{F}) = \argmax_{\pi \in \mathcal{F}} \sum_{x \in \mathcal{X}} \sum_{y \in \mathcal{Y}} \pi(y|x) \hat{r}(x, y; D_\succsim^\text{ordinal})
\]
where 
\begin{align*}
    D_\succsim^\text{ordinal} & = \{(x_i, y_i, y_i', \mathbbm{1}\{y_i \succsim_{x_i} y_i'\})\}_{i=1}^N \\
    \hat r(x, y; D_\succsim^\text{ordinal})
& \;\in\;
\arg\min_{r}\!
\sum_{(x_i,y,y')\in D_{\succsim}^{\text{ordinal}}}
\log\sigma\!\Bigl(
\bigl[2\cdot 1\{y \succsim_{x_i} y'\}-1\,\bigr]\,[\,r(x_i,y)-r(x_i,y')\,]
\Bigr)
\end{align*}

\section{Insufficiency of Ordinal Data}
In this section, we outline the central theoretical result of the paper. We show that no preference fine-tuning algorithm that only uses ordinal feedback can always select the most preferable, feasible model. This impossibility result holds for all methods that rely on ordinal data such as DPO and RLHF -- but also ODPO, REBEL, and others.
\begin{theorem}\label{ordinsufficient}
    Let $D_\succsim^\text{ordinal} = \{(x_i, y_i, y_i', \mathbbm{1}\{y_i \succsim_x y_i'\})\}_{i=1}^N$ for any $N$. Then there does not exist an algorithm $f$ such that for any feasible set of language models $\mathcal{F}$ and any possible preference $\succsim \in \mathcal{P}(\Pi)$, it outputs the most preferred model $f(D_\succsim^\text{ordinal}, \mathcal{F}) \succsim \pi$ for all $\pi \in \mathcal{F}$. 
\end{theorem}

The intuition for this result is that there is an identification problem,  $\succsim$ is not identified from any dataset of ordinal feedback, which prevents any algorithm from systematically selecting the optimal model according to $\succsim$.\footnote{The proof constructs a case of identification failure. } Returning to our Figure 1 example, the ordinal wins alone do not reveal that model A is better as they do not communicate that prompt $1$ is a critical medical prompt that is crucial to avoid hallucinations on. We need to know how $\succsim$ tradeoffs performance on different prompts and response quality differences, but that information is not in the ordinal feedback. 

Furthermore, we show this tradeoff problem is not just a theoretical concern. During a DPO training run, we track individual sample losses across $100$ validation set points. Figure \ref{fig:init_heatmap} reveals clear tradeoffs across prompts, with performance consistently degrading throughout training for some samples -- suggesting that standard preference fine-tuning methods are implicitly making tradeoffs behind the scenes. However, these tradeoffs are being made without any data on how we ought to make tradeoffs, it is simply being done arbitrarily. 

\begin{figure}[h!]
    \centering
    \includegraphics[width=0.5\columnwidth]{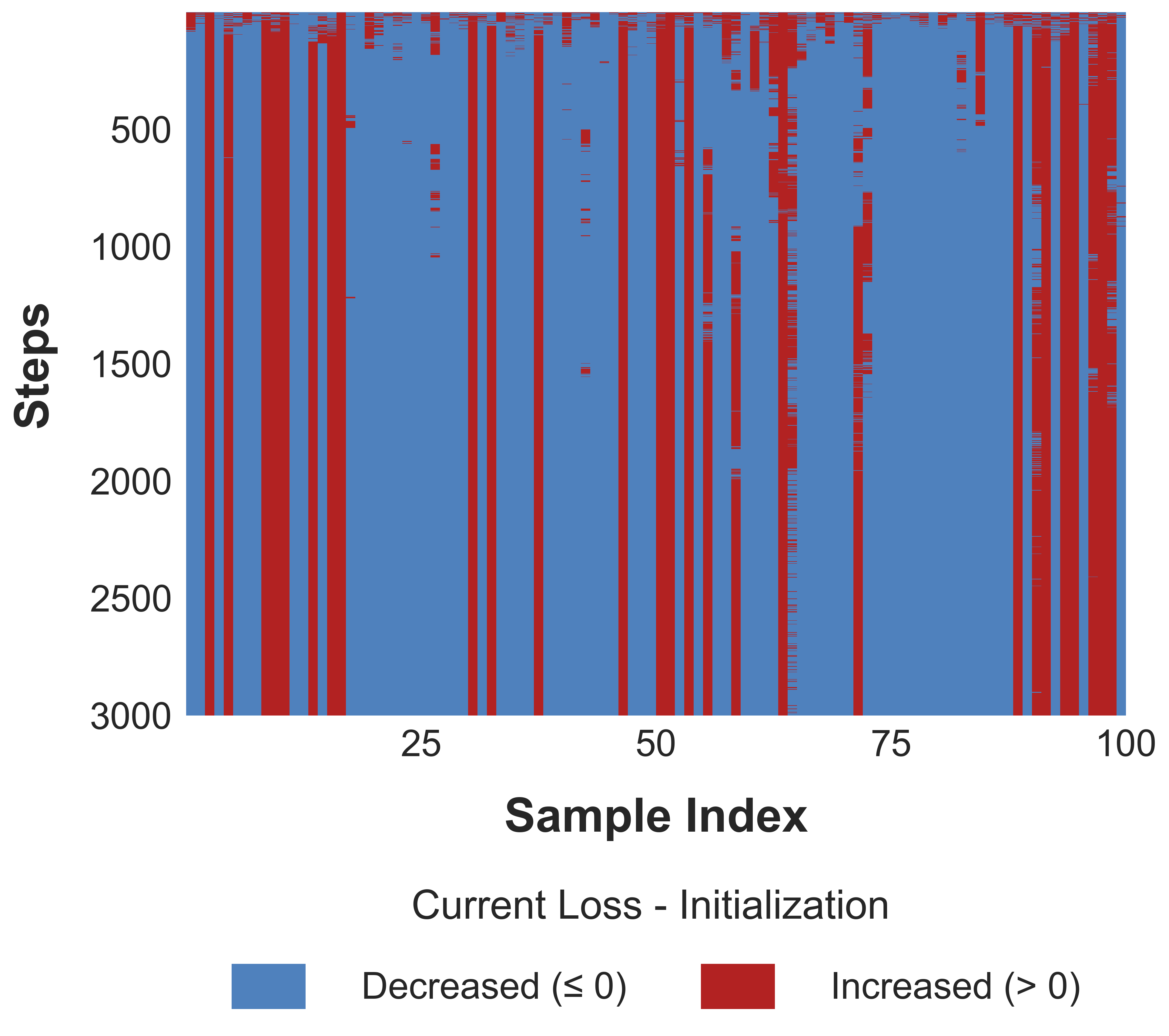}
    \caption{\textbf{DPO optimization performs sample-level tradeoffs.} Individual sample-level change in loss versus initialization in DPO. Blue indicates improvement in loss, red indicates degradation. While overall DPO loss decreases, 27\% of prompts have a higher loss than at initialization. This indicates consistent sample-level alignment tradeoffs being made during optimization.}
    \label{fig:init_heatmap}
\end{figure}

One principled way to navigate these tradeoffs is by using cardinal utility, such as our example in Figure $1$. Since $100 > .2 + .2$, then it is worth trading-off the second two prompts for the first.

\subsection{Against Using Bradley-Terry for Cardinal Utility}

In RLHF/DPO, instead of eliciting cardinal values directly from labelers to navigate these tradeoffs, cardinal values are inferred from ordinal values via Bradley-Terry. In this section, we argue that Bradley-Terry is unlikely to be accurately capturing these cardinal values. 

The Bradley-Terry model assumes that $\succsim_{x}$ are stochastic and the following equation holds
\[
p(y_2 \succsim_{x} y_1) = \frac{1}{1 + e^{r(x, y_1) - r(x, y_2) }}
\]
Therefore, it assumes we can tell how big an improvement $y_1$ is over $y_2$ by observing the fraction of samples where $y_2$ is picked over $y_1$.

From a conceptual perspective, choice frequencies can severely misrepresent the actual importance of improvements. Consider two scenarios that both generate near 100\% choice frequencies: first, where $y_1$ is identical to $y_2$ but without an obvious spelling error, and second, where $y_1$ provides accurate health information while $y_2$ contains dangerously incorrect medical advice. The Bradley-Terry model would assign equal importance to these improvements since they produce identical choice frequencies. Yet there is clearly a vast difference in the actual significance of these improvements that should be reflected in our preference strength measure.

Even if in theory one could recover cardinal values by repeatedly sampling preferences between the same response pairs, real preference datasets typically contain only a single comparison between any $y_1$ and $y_2$ for a given prompt $x$. With a single data-point, it is impossible to recover the fraction of times $y_1$ would be chosen over $y_2$ if we repeated sampling. 

Finally, our experimental results in Section $5.1$ provide empirical evidence that the Bradley-Terry parametric assumption does not hold, and that they do a poor job of predicting cardinal feedback provided directly via labelers. 

\subsection{Eliciting Cardinal Utility via Willingness to Pay}
Suppose we want to keep the maximization of mean reward specification given by RLHF/DPO, but we want to avoid theorem \ref{ordinsufficient}. Our strategy is to elicit cardinal data directly from labelers so we can infer cardinal utility  $r(x, y)$. 

This cardinal data will determine what tradeoffs to make. Crucially, if $r(x, y') - r(x, y) = 5$ then we are literally saying it is worth exactly $5$ improvements worth $r(x, y') - r(x, y) = 1$. Basic discrete Likert scales do not satisfy this property as $1/5$ is not meaningfully $5$ times less good than $5/5$. 

Instead, we instead need a numeric that is comparable across prompts and responses whose scale is cardinally meaningful. One such numeric is money. It is a well-established tradition in economics to use willingness-to-pay (WTP) to elicit people's cardinal utility for various goods. The advantage of money is that it is cardinally meaningful- $\$5$ is really worth $5$ times more than $\$1$.\footnote{Assuming that utility is quasi-linear in money, which we will assume throughout the paper. } Moreover, its scale offers a natural calibration across labelers. By Appendix Theorem \ref{affine}, it is without loss to multiply/divide $r$ by any constant, so we can divide by the standard deviation of each labeler's values to calibrate. Money also has a side benefit of being intuitive and familiar to labelers.

Therefore, we propose a small modification to RLHF/DPO. Instead of just collecting $D_\succsim^\text{ordinal}$, collect \[
D_\succsim^\text{WTP} = \{x_i, y_i, y_i', w_i\}_{i = 1}^N
\] where $w_i$ denotes the WTP to go from response $y_i$ to $y_i'$ in response to prompt $x_i$.\footnote{In an ideal setting, you could actually collect money, but for our case, we use simple stated WTP. } Then instead of choosing $r$ to minimize cross-entropy, we minimize squared error distance from $w_i$. 
\begin{equation}\label{rmin}
    \hat{r}(x_i, y_i; D_\succsim^\text{WTP}) \in \argmin_{r} \sum_{i =1 }^N \bigg(r(x_i, y_i') - r(x_i, y_i) - w_i\bigg)^2
\end{equation}

Finally, we maximize mean reward. Call this Cardinal Reinforcement Learning from Human Feedback (CRLHF). The following shows that in the data limit, CRLHF will always select the optimal model. 
\begin{theorem}\label{sufficient}
    Let $D_\succsim^\text{WTP}$ and suppose $N$ is large enough to cover all $x \in \mathcal{X}, y \in \mathcal{Y}, y' \in \mathcal{Y}.$ Let 
    \begin{equation}\label{max}
        f(D_\succsim^\text{WTP}, \mathcal{F}) =  
    \argmax_{\pi \in \mathcal{F}} \sum_{x \in \mathcal{X}} \sum_{y \in \mathcal{Y}} \pi(y|x)
    \hat{r}(x_i, y_i; D_\succsim^\text{WTP})
    \end{equation}
    Then, for any possible preferences $\succsim \in \mathcal{P}(\Pi)$ and any feasible set of language models $\mathcal{F}$, $f(D_\succsim^\text{WTP}, \mathcal{F}) \succsim \pi$ for all $\pi \in \mathcal{F}$. In other words, RLHF with a reward trained via MSE loss on willingness-to-pay estimates always produces the most preferred model in the data limit. 
\end{theorem}

Alternatively, there is a DPO style objective that solves for the same policy as above under some assumptions \citep{rafailov2024direct}. 
\begin{equation}
\label{CDPO}
\begin{aligned}
\argmin_{\pi^*} \sum_{D_\succsim^\text{WTP}} \Bigl(\beta \log \tfrac{\pi^*(y_i'|x_i)}{\pi_{\text{ref}}(y_i'|x_i)}
        - \beta \log \tfrac{\pi^*(y_i|x_i)}{\pi_{\text{ref}}(y_i|x_i)} - w_i\Bigr)^2
\end{aligned}
\end{equation}

Then under the same conditions as those listed by \citet{rafailov2024direct}, the solution to Equation \ref{max} and Equation \ref{CDPO} must be the same. We call this Cardinal Direct Preference Optimization (CDPO).\footnote{This regression-based procedure matches \cite{gao2024rebel}. The difference is we recommend optimizing against cardinal feedback from labelers directly, instead of from a reward model estimated via Bradley-Terry. }

\section{Experiments}

\subsection{\textsc{CardinalPrefs} Dataset and Cardinal Label Quality} \label{subsec:cardinal_label_quality}

A natural concern with using cardinal data is that it is too noisy or difficult for labelers to provide, and inconsistent across labelers \citep{christiano2017deep}. 
In theory, our approach partially alleviates these concerns because labels are still relative $e.g., $ $y$ vs $y'$ and because money gives labelers a familiar, clear scale to use. In this section, we check empirically if our WTP scheme can elicit high-quality cardinal data at scale, with an eye on performance relative to ordinal data with Bradley-Terry.

We elicit WTP values from five labelers across over $25,000$ tuples of prompts and response pairs chosen from Anthropic's HHH dataset \citep{bai2022training} and Chatbot Arena comparisons \citep{chiang2024chatbot} to form the \textsc{CardinalPrefs} dataset. Anecdotally, the labelers reported that providing WTP made each sample took roughly $6-10\%$ longer than purely ordinal labels. With this data in hand, we perform 3 data quality checks. 

\begin{figure}[h]
    \centering
    \includegraphics[width=0.60\columnwidth]{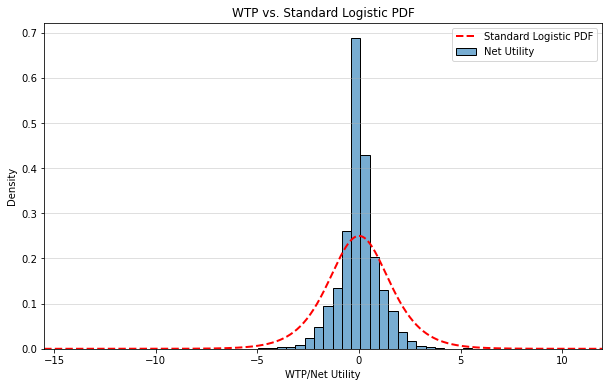}
    \caption{Distribution of WTP Values}
    \label{WTP_dist}
\end{figure}

\textbf{Figure \ref{WTP_dist} shows that the distribution of cardinal labels is smooth across multiple labelers, but does not match Bradley-Terry assumptions.} Aggregating cardinal data across multiple labelers introduces the possibility of improper normalization across raters. However, in Figure \ref{WTP_dist}, the distribution after normalization is unimodal with fast decaying tails, indicating that in aggregate these problems do not appear. Second, Bradley-Terry implies the distribution of the cardinal values should be logistic, but empirically we can see actual people's cardinal values are more peaked and have faster decaying tails. This is empirical evidence against the parametric assumption in Bradley-Terry. 

\textbf{Figure \ref{fig:utility-examples} shows  two randomly selected draws from the dataset which are chosen to be representative cases of "high" and "low" labeled WTP - we find intuitively that the high WTP case is much more important. } In the high WTP case, the dispreferred response does not even answer the question while the preferred response gives a complete answer. In the low WTP case both answers respond to the prompt accurately but the preferred response is more formal and informative. This gives some sanity check that labelers can provide cardinal data that is meaningful. 

\begin{figure*}[h!]
\centering
\begin{tabular}{|p{0.15\textwidth}|p{0.4\textwidth}|p{0.4\textwidth}|}
\hline
& \multicolumn{1}{c|}{\textbf{High Willingness to Pay (6.34)}} & \multicolumn{1}{c|}{\textbf{Low Willingness to Pay (0.13)}} \\
\hline
\textbf{Prompt} & What are some common data structures? & What do people mean when they talk about the "nuclear family"? \\
\hline
\textbf{Preferred} & \small Data structures are fundamental concepts in computer science and programming. (...)

1. \textbf{Arrays}: A collection of elements identified by index or key.
2. \textbf{Linked Lists}: A linear collection of nodes with data and references.
3. \textbf{Stacks}: LIFO data structures with push/pop operations. (...)
& \small The nuclear family is a term used to describe an intimate family unit consisting of a married couple and their dependent children. This term derives from the description of the central part of the atomic structure of an atom. (...) \\
\hline
\textbf{Dispreferred} & \small \#NAME? & \small Oh that's a broad topic! The "nuclear family" generally refers to the traditional nuclear family that most people grew up in, consisting of a mother and a father, and children. (...) \\
\hline
\end{tabular}
\caption{An example of a high and low WTP labels. These examples are randomly chosen from the dataset filtered to high WTP and low WTP cases respectively. }
\label{fig:utility-examples}
\end{figure*}

As our final data-quality check, we compute the implicit reward estimate from the cardinal data directly via CDPO and via Bradley-Terry via DPO. With these values, we compute the mean-squared error (MSE) of the implicit reward differences vs. the WTP data on a held-out set. If WTP data only added noise and no signal, we would expect similar MSE from training on ordinal data and on the cardinal data. Similarly, if Bradley-Terry accurately captured preference intensity already in ordinal data, then we would also expect similar MSE across both methods. \textbf{We find the cardinal data's MSE is $\textbf{55.92}\%$ smaller than the ordinal data's, indicating meaningful signal increase over ordinal data with Bradley-Terry.}

\subsection{DPO vs CDPO Experiments}

In this section we empirically test if adding in cardinal data boosts performance in the preference fine-tuning pipeline. We focus on comparing DPO and CDPO as it is computationally cheaper than RLHF and CRLHF. A central challenge with these experiments is evaluation. Our goal is to optimize $\succsim$, so the natural benchmark is $\succsim$; however, $\succsim$ is very difficult/expensive to measure. For example, it may take a labeler weeks to decide if they prefer GPT-$4$o or Grok $3$. Therefore, we utilize a wide variety of experiments to triangulate the effectiveness of using cardinal data. 

\subsection*{Experiment 1: Real Data, Simplified Training Setting}

Our first experiment involves a simplified setting where the set of feasible models is small enough such that we can elicit model-level preferences from labelers. This allows us to directly measure if using WTP data makes the resulting model more preferable. We find exploiting cardinal data leads to significantly higher probability of the resulting model being more preferable. 

Let $\pi_\theta \coloneq \pi^{(1)}\theta + \pi^{(2)}(1 - \theta)$ and $\mathcal{F}= \{\pi_\theta \mid \theta \in [0, 1]\}$ where $\pi^{(1)}$ and $\pi^{(2)}$ only differ on $3$ prompts. Because $\pi^{(1)}$ and $\pi^{(2)}$ only differ on three prompts, it is feasible to elicit from labelers if $ \pi^{(1)} \succsim \pi^{(2)}$, their model-level preferences. 

The experiment proceeds as follows. First, we hired a set of multiple experienced labelers. Second, we had the labelers give both ordinal feedback and WTP at the prompt-level. Third, we set up objective functions for DPO following \citet{rafailov2024direct} and CDPO following Equation \ref{CDPO} on the prompt-level feedback\footnote{We set the reference model as $\pi_\frac{1}{2}$.}. Optimization occurs along a single parameter $\theta$, and can be solved via gradient descent. Fourth, we elicit model-level preferences from the labeler between $\pi^{(1)}$ and $\pi^{(2)}$ e.g., if $ \pi^{(1)} \succsim \pi^{(2)}$. Finally, we compare the selected models by DPO and CDPO to the labeler's preferences. We say CDPO/DPO selects the optimal model if it puts more weight on the preferred model e.g., $\pi_\theta$ has $\theta > \frac{1}{2}$ where $\pi^{(1)} \succsim \pi^{(2)}$. 

\begin{table}[h]
\centering

\caption{CDPO Selects the Optimal Model More Often than DPO}\label{tab:select_optimal}
\label{tab:comparison}
\small
\begin{tabular}{lcc}
\toprule
        & \multicolumn{2}{c}{Selects Optimal Model}\\
\cmidrule(lr){2-3}
\textbf{Method} & Full sample & Where CDPO and DPO Disagree\\
\midrule
CDPO   & 90.27\% & 91.67\%\\
       & (1.48\%) & (3.99\%)\\[4pt]
DPO    & 83.29\% & 33.33\%\\
       & (1.86\%) & (6.80\%)\\[4pt]
\midrule
Diff.  & 6.98\%*** & 58.33\%***\\
       & (1.44\%) & (9.24\%)\\
\bottomrule
\end{tabular}
\begin{flushleft}   
\small\emph{Notes:} $N=401$ (full sample), $N=48$ (sample where DPO and CDPO disagree).
$^{*}p<0.1$, $^{**}p<0.05$, $^{***}p<0.01$. “Selects Optimal Model’’ is the percentage of times the method chooses a model weakly preferred to the alternative; standard errors in parentheses.
\end{flushleft}
\end{table}

\textbf{Table \ref{tab:select_optimal} shows that CDPO selects the optimal model significantly more often than DPO.} The first column shows that CDPO selects the optimal model almost $7$ percentage points more often than DPO. However, this masks many cases where CDPO and DPO select the same model. The second column shows CDPO outperforms DPO by over $58$ percentage points in the sample where CDPO and DPO disagree.

This experiment shows that, taking into account all the noise, heterogeneity, etc. in cardinal feedback, on net it helps the fine-tuning algorithm output the preferred model significantly more often. 

\subsection*{Experiment 2: Simulated Data, Real Training}
In this experiment, we use simulated preference data from a ground-truth reward model in the standard Chatbot fine-tuning setting. We find that CDPO results in $50\%$ higher mean ground-truth reward than DPO. 

The experiment proceeds as follows. First, we take simulated preferences set by ground-truth reward model \citep{liu2024skywork}. We use a dataset of pairs of responses from Ultrafeedback-200k, a large preference dataset covering a broad suite of chat and reasoning tasks \citep{cui2024ultrafeedbackboostinglanguagemodels}. We then compute the ground-truth rewards for each response $r^{GT}(x,y_1)$ and $r^{GT}(x,y_2)$. For CDPO, we label directly with the reward margin $r^{GT}(x,y_1) - r^{GT}(x,y_2)$. For DPO, we label directly with the sign of $r^{GT}(x,y_1) - r^{GT}(x,y_2)$. Second, we replicate the Zephyr fine-tuning recipe \citep{tunstall2023zephyrdirectdistillationlm}. As our base model, we use a Mistral-7B model \citep{huggingface_mistral7b_sft_beta} trained with supervised fine-tuning on the UltraChat dataset \citep{ding2023enhancingchatlanguagemodels}. Third, we set up objective functions for DPO following \citet{rafailov2024direct} and CDPO following Equation \ref{CDPO}. Let $\pi_{\text{DPO}}$ and $\pi_{\text{CDPO}}$ denote the maximizer we get from DPO and CDPO. We fine-tune the Mistral-7B model with LoRA for one epoch (see Appendix \ref{appendix:training} for details). Fourth, we score generated text to estimate mean reward. 

\begin{figure*}[h!]
    \centering
    \begin{minipage}[t]{0.48\textwidth}
        \centering
        \includegraphics[width=\linewidth]{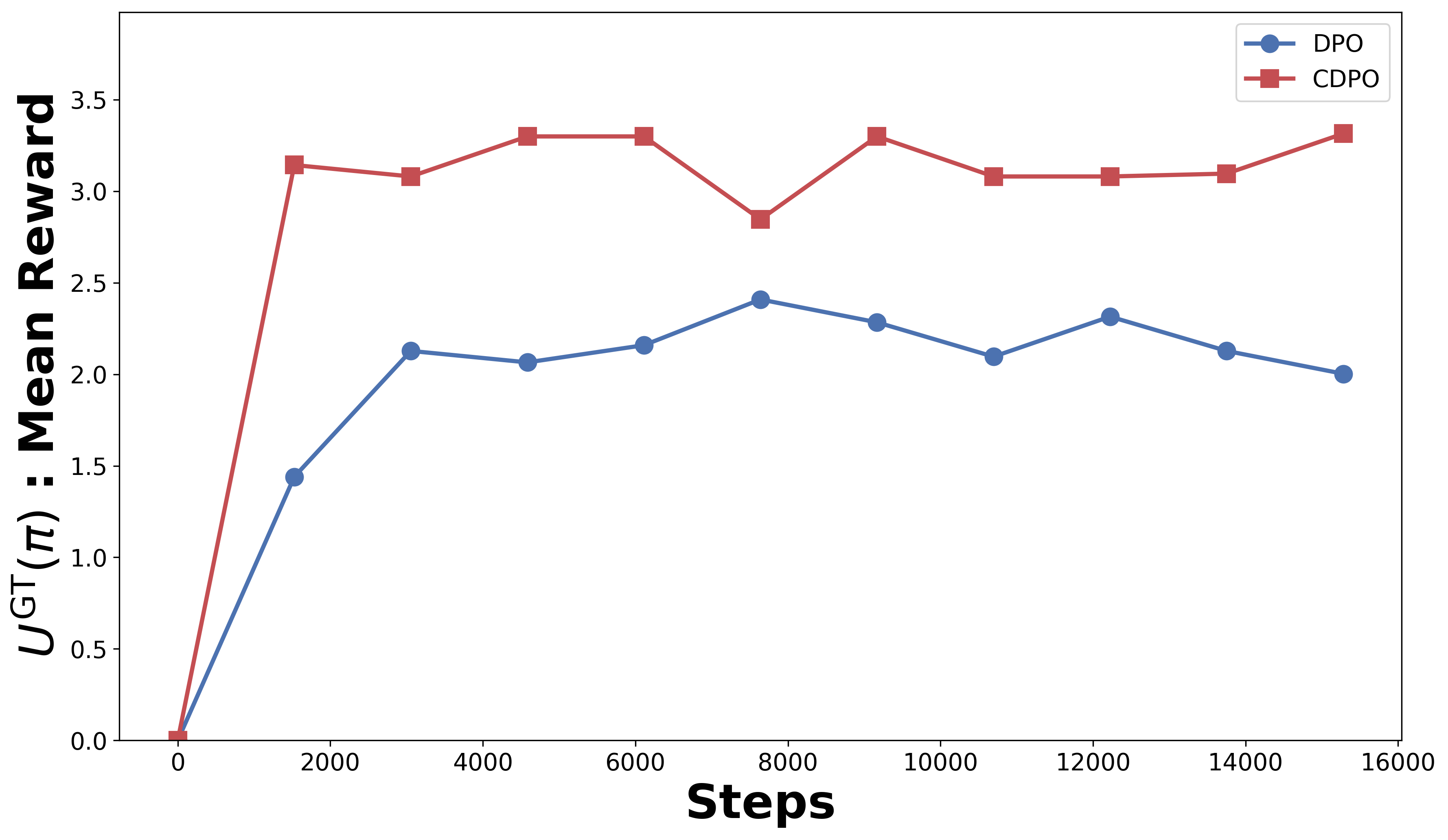}
        \label{fig:mean_reward}
    \end{minipage}
    \hfill
    \begin{minipage}[t]{0.48\textwidth}
        \centering
        \includegraphics[width=\linewidth]{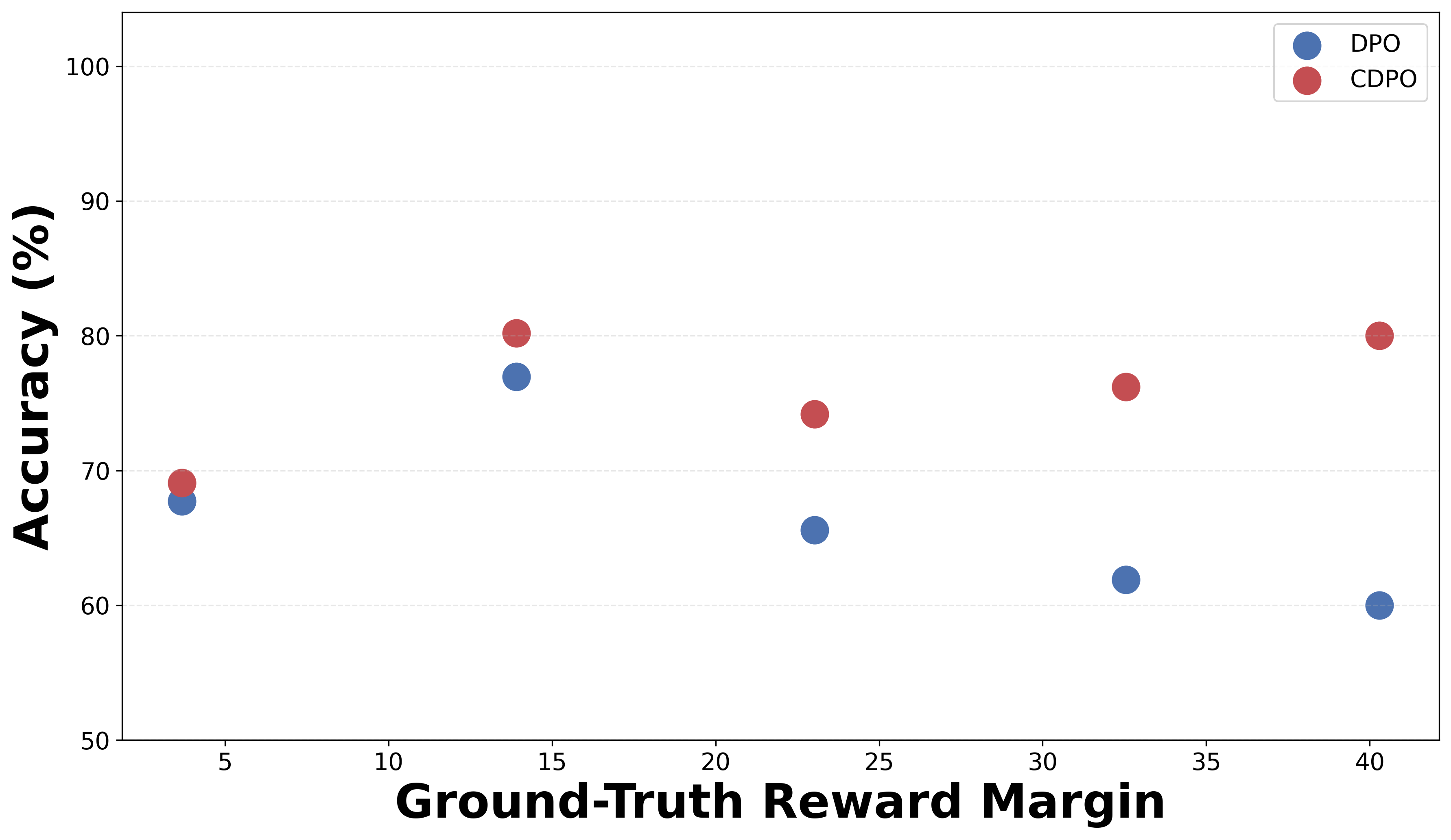}
        \label{fig:margin_agreement}
    \end{minipage}
    \caption{\textbf{CDPO achieves higher utility by optimizing for preference strength.} 
    \textbf{Left:} CDPO achieves higher utility than DPO, where utility measures mean alignment of generations with human preferences.
    Utility is normalized such that the base model gets $0$ utility.
    \textbf{Right:}
    CDPO's advantage over DPO increases with preference strength.
    }\label{fig:utility_accuracy}
\end{figure*}

\textbf{CDPO achieves higher utility than DPO.} Figure \ref{fig:utility_accuracy} shows the central result: CDPO achieves considerably higher utility than DPO, and thus we empirically can confirm that CDPO is producing more preferred/more aligned models than DPO by taking into account richer preference data. Notice that CDPO does better at $2,000$ steps than DPO does at over $14,000$ steps. This indicates that even if collecting cardinal data is more expensive, and we can only collect $\frac{1}{10}$ as much data, CDPO would still be better. 

\textbf{CDPO prioritizes getting strong preferences right.} To gain some intuition about why CDPO is producing more preferable models than DPO, recall that during the optimization process, tradeoffs between prompt-response pairs are inevitable. CDPO makes these tradeoffs by using cardinal human feedback to prioritize high-impact improvements, while DPO gives equal weight to all comparisons regardless of importance. Figure \ref{fig:margin_agreement} demonstrates this quantitatively: analyzing a validation set of 1000 tuples $(x, y_1, y_2)$ stratified by reward margin ($r^{\text{GT}}(x, y_2) - r^{\text{GT}}(x, y_1)$), we find that while the two algorithms perform comparably on low-margin cases, CDPO excels on medium and especially high-margin cases. This shows CDPO successfully prioritizes cases humans identify as important, while DPO wastes optimization effort on relatively unimportant improvements.

\subsection*{Experiment 3: Real Data, Real Training}

In this experiment, we use both real-world cardinal data and the standard chatbot preference fine-tuning setting. Now we cannot observe $\succsim$ directly, so we rely on various preference proxy measures. 

The experiment proceeds as follows: we take the exact same training procedure as in experiment 2 but instead of using Ultrafeedback-200k with simulated preference data, we use the real-world \textsc{CardinalPrefs} preference dataset we gathered in Section \ref{subsec:cardinal_label_quality}. 

\textbf{On real-world preference measures, CDPO exceeds DPO by prioritizing important cases.} On a held-out set, CDPO outputs the preferred response at roughly the same rate as DPO. This is expected as CDPO does not magically make it possible to get more prompts right. However, CDPO does apply more optimization pressure to "important prompts." When we weight observations by WTP, we see CDPO does about $4$ percentage points better. Similarly, when we weight observations by importance as determined by Claude 3.7 Sonnet, we also see CDPO scores higher by $3$ percentage points, indicating that CDPO is properly prioritizing the important cases. Similarly, when we filter to observations that where a response is substantively better as determined by Claude, we also see CDPO scores higher by about $2.5$ percentage points but on stylistic issues, DPO exceeds by around $9$ percentage points. This indicates that CDPO successfully identifies and applies more optimization pressure to important issues, while DPO wastes optimization pressure on relatively less important, stylistic improvements because ordinal measures cannot properly prioritize based on importance. 

Finally, on Arena-Hard, a state-of-the-art preference evaluation measuring win-rates against GPT-4, CDPO wins almost $55\%$ more battles than DPO. 

\begin{table}[h!]
    \centering
    \caption{CDPO exceeds DPO by prioritizing important cases.}
    \begin{tabular}{lcc}
        \toprule
        \textbf{Metric} & \textbf{DPO} & \textbf{CDPO} \\
        \midrule
        Outputs Preferred Response & 51.1\% & \textbf{51.6\% } \\
        Outputs Preferred Response (WTP Weighted) & 52.8\% & \textbf{56.7\%} \\
        Outputs Preferred Response (Importance Weighted) & 46.0\% & \textbf{49.0}\% \\
        Outputs Preferred Response (Substantive) & 44.8\% & \textbf{47.2}\% \\
        Outputs Preferred Response (Style) & \textbf{66.7}\% & 55.6\% \\
        Arena-Hard Winrate & 3.80\% & \textbf{5.9\%} \\
        \bottomrule
    \end{tabular}
    \label{tab:model_comparison}
\end{table}

\section{Conclusion}

In this paper, we demonstrated fundamental limitations of ordinal human feedback and explored the benefit of collecting richer, cardinal feedback. We found that using cardinal feedback resulted in models that are rated as more preferable and score higher on preference-based benchmarks. In future work, it would be valuable to study and improve cardinal feedback elicitation protocols and interface design in an LLM context such as WTP with a fixed budget, implied cardinal ratings from lottery choices, among others. 


\newpage
\clearpage

\section*{Impact Statement}
This research aims to improve language model alignment by developing more theoretically grounded methods for learning and modeling user preferences. While learning people's preferences more accurately could provide significant benefits - including enhanced safety, reliability, and utility - it is important to note that this improvement alone does not address all AI safety concerns.

\bibliography{research} 
\bibliographystyle{plainnat}

\appendix
\onecolumn

\section{Proofs}
\renewenvironment{proof}[1][\proofname]{{\bfseries #1.}}{\qed}

\begin{proof}[Proof of Theorem \ref{ordinsufficient}]
    Suppose there exists an $f$ such that for any $\mathcal{F}$ and any $\succsim \in \mathcal{P}(\Pi)$, $f(D_\succsim^\text{ordinal}, \mathcal{F}) \succsim \pi$ for all $\pi \in \mathcal{F}$. 

    Let $\mathcal{F} = \{\pi^{(1)}, \pi^{(2)}\}$ where $\pi^{(1)}$ and $\pi^{(2)}$ differ only on two prompts, $x_1$ and $x_2$ where $\pi^{(1)}$ is better than $\pi^{(2)}$ on prompt $x_1$ and the reverse holds for prompt $x_2$. 
    
    Let $\succsim^{(1)}, \succsim^{(2)} \in P(\Pi)$ such that $\pi^{(1)} \succ^{(1)} \pi^{(2)}$,  $\pi^{(2)} \succ^{(2)} \pi^{(1)}$ but also such that $\succsim_{x}^{(1)} = \succsim_{x}^{(2)}$ for all $x \in \mathcal{X}$ i.e. both preferences agree on which responses are better for each prompt, but they weigh the prompts differently. 
    
    Then, $D_{\succsim^{(1)}}^\text{ordinal} = D_{\succsim^{(2)}}^\text{ordinal}$ (i.e. $\succsim$ is not identified from $D^\text{ordinal}$) and therefore
    \[\pi^* = f(D_{\succsim^{(1)}}^\text{ordinal}, \mathcal{F}) = f(D_{\succsim^{(2)}}^\text{ordinal}, \mathcal{F})
    \]    

    Without loss of generality, suppose $\pi^* = \pi^{(1)}$. Then we have $f(D_{\succsim^{(2)}}^\text{ordinal}, \mathcal{F}) = \pi^{(1)} \not \succsim^{(2)} \pi^{(2)}$. This violates our assumption that $f$ is such that $f(D_\succsim^\text{ordinal}, \mathcal{F}) \succsim \pi$ for all $\pi \in \mathcal{F}$ and all $\succsim \in P(\Pi)$. 
\end{proof}

\begin{proof}[Proof of Theorem \ref{sufficient}]
    By Theorem \ref{kreps}, $\succsim$ admits a utility function of the form $\sum_{x \in \mathcal{X}} \sum_{y \in \mathcal{Y}} \pi(y|x) r^{\text{true}}(x, y)$. Under quasi-linearity of money, the WTP to go from $y$ to $y'$ as a response to $x$ is given by $r(x, y') - r(x, y)$. That implies the reward modeling stage solves
    \[
    \hat{r}(x_i, y_i; D_\succsim^\text{WTP}) \in \argmin_{r} \bigg(r(x_i, y_i') - r(x_i, y_i) - r^\text{true}(x_i, y_i') - r^\text{true}(x_i, y_i)\bigg)^2
    \]
    Since $D_\succsim^\text{WTP}$ is sufficiently large, $\hat{r}(x_i, y_i; D_\succsim^\text{WTP})$ will perfectly match $r^\text{true}(x_i, y_i)$ up to an affine transformation. Since $r$ is unique up to affine transformations by Theorem \ref{affine}, $\sum_{x \in \mathcal{X}} \sum_{y \in \mathcal{Y}} \pi(y|x)
    \hat{r}(x_i, y_i; D_\succsim^\text{WTP})$ is also a utility function that represents $\succsim$. Since $\mathcal{F}$ is compact by assumption, the argmax exists and is always the most preferred feasible policy i.e. $f(D_\succsim^\text{WTP}, \mathcal{F}) \succsim \pi$ for all $\pi \in \mathcal{F}$. 
\end{proof}

\subsection{Additional Appendix Results}
\begin{theorem}\label{kreps}
    $\succsim$ satisfies completeness, transitivity, continuity and independence if and only if there exists a $r: \mathcal{X} \times \mathcal{Y} \to \mathcal{R}$ such that $\sum\limits_{x \in \mathcal{X}} \sum\limits_{y \in \mathcal{Y}} \pi(y|x) r(x, y)$ represents $\succsim$. 
\end{theorem}
\begin{proof}[Proof of Theorem \ref{kreps}]
    
    This proof follows by exploiting a mathematical equivalence between language models and subjective expected utility \citep{anscombe1963definition}. 

     Notice that our setting is mathematically identical to subjective expected utility with state dependent preferences. Consider the following analogy:
\begin{enumerate}
    \item In our context:
    \begin{itemize}
        \item We have a finite number of prompts
        \item We have a finite number of responses 
        \item We have a language model $\pi$ which maps from our prompts to lotteries over our responses
    \end{itemize}
    
    \item In subjective expected utility 
    \begin{itemize}
        \item We have a finite number of states 
        \item We have a finite number of outcomes 
        \item We have a `act' $f$ which maps from the states to lotteries over outcomes. 
    \end{itemize}
\end{enumerate}
Usually subjective expected utility analysis has extra assumptions added on that will break the analogy, like the assumption of state-independent preferences. In our setting, that would be equivalent to prompt-independent preferences, which is not plausible.

But if we focus on the general case of subjective expected utility with state-dependent preferences, our language model setting exactly lines up with the mathematical setting considered in \citet{kreps1988notes} proposition 7.4. This proposition shows in the subjective expected utility setting that completeness, transitivity, continuity and independence are necessary and sufficient for an additively separable, expected utility representation. Therefore, we have our result. 
\end{proof}

\begin{theorem}\label{affine}
    $\sum_{x \in \mathcal X}\sum_{y \in \mathcal Y}\pi(y\mid x)\,r^{(1)}(x,y)$
    and
    $\sum_{x \in \mathcal X}\sum_{y \in \mathcal Y}\pi(y\mid x)\,r^{(2)}(x,y)$
    represent the same $\succsim$
    iff there exist $a>0$ and $b:\mathcal X\to\mathbb R$ such that  
    $r^{(1)}(x,y)=a\,r^{(2)}(x,y)+b(x)$ for all $x,y$.
\end{theorem}
\begin{proof}[Proof of Theorem \ref{affine}]
Define
\[
  u^{(i)}(\pi)=\sum_{x,y}\pi(y\mid x)\,r^{(i)}(x,y),\qquad i=1,2.
\]

\paragraph{$(\Rightarrow)$}
Assume $u^{(1)}$ and $u^{(2)}$ induce the same order.  
Hence $\exists$ a strictly increasing $f:\mathbb R\to\mathbb R$ with  
\(u^{(1)}(\pi)=f\bigl(u^{(2)}(\pi)\bigr)\) for all $\pi$.

For any policies $\pi^1,\pi^2$ and any $\lambda\in[0,1]$ define  
\[
  \pi^\lambda(y\mid x)
  :=\lambda\,\pi^1(y\mid x)+(1-\lambda)\,\pi^2(y\mid x),
  \quad\forall x,y.
\]
$u^{(i)}$ is given as linear in $\pi(y|x)$
\[
  u^{(i)}(\pi^\lambda)=\lambda\,u^{(i)}(\pi^1)+(1-\lambda)\,u^{(i)}(\pi^2),
  \qquad i=1,2. \tag{1}
\]

Let  
\(s:=u^{(2)}(\pi^1)\) and \(t:=u^{(2)}(\pi^2)\).  
Equation (1) and the representation relation give
\[
  f\!\bigl(\lambda s+(1-\lambda)t\bigr)
  \;=\;
  f\!\bigl(u^{(2)}(\pi^\lambda)\bigr)
  \;=\;
  u^{(1)}(\pi^\lambda)
  \;\overset{(1)}{=}\;
  \lambda\,u^{(1)}(\pi^1)+(1-\lambda)\,u^{(1)}(\pi^2)
  \;=\;
  \lambda f(s)+(1-\lambda)f(t).
\]

Because $s$ and $t$ range over all values $u^{(2)}$ attains, and
$\lambda s+(1-\lambda)t$ is again in that range (by $\pi^\lambda$),
The above equation holds on a convex subset of $\mathbb R$.
A strictly increasing real function satisfying (2) on a convex set must be affine,  
so \(f(t)=a\,t+b\) with \(a>0\).

\medskip
Substituting back, for every $\pi$
\[
  \sum_{x,y}\pi(y\mid x)\,\bigl[r^{(1)}(x,y)-a\,r^{(2)}(x,y)\bigr]=b.
\]
Fix any \(x\) and view the left side as an affine functional of
\(\pi(\cdot\mid x)\) on the simplex \(\Delta(\mathcal Y)\);
being constant on that simplex forces the coefficient vector to be constant:
\[
  r^{(1)}(x,y)-a\,r^{(2)}(x,y)=b(x)\quad\forall y,
\]
hence \(r^{(1)}(x,y)=a\,r^{(2)}(x,y)+b(x)\).

\paragraph{$(\Leftarrow)$}
Conversely, if \(r^{(1)}(x,y)=a\,r^{(2)}(x,y)+b(x)\) with \(a>0\), then
\[
  u^{(1)}(\pi)=a\,u^{(2)}(\pi)+\sum_x b(x),
\]
a positive‑affine transform of \(u^{(2)}\); therefore the two
functionals induce exactly the same $\succsim$.
\end{proof}

\section{Additional Theoretical Results}
\subsection{Deriving $\succsim_x$}
In this subsection, we formally show how to derive $\succsim_x$ from $\succsim$. 

Given $\pi \in \Pi$, define 
\[
\pi_{x_j, \tilde{y}}(x) = 
\begin{cases} 
\pi(x) & \text{if } x \neq x_j \\
\tilde{y}  & \text{if } x = x_j
\end{cases}
\]
That is, $\pi_{x_j, \tilde{y}}$ is $\pi$ where we modify just how it responds to $x_j$ to $\tilde{y}$. 

With this machinery, we can define preferences over responses. 
    Let $\tilde{y}^{(1)}, \tilde{y}^{(2)} \in \Delta(\mathcal{Y})$. Define $\succsim_{x_i}$ as follows 
    \begin{equation*}
    \tilde{y}^{(1)} \succsim_{x_i} \tilde{y}^{(2)}
    \overset{\text{def}}{\iff} \pi_{x_i, \tilde{y}^{(1)}} \succsim \pi_{x_i, \tilde{y}^{(2)}} \quad \forall \pi \in \Pi
\end{equation*}

That is, distribution of responses $\tilde{y}^{(1)}$ is preferred to $\tilde{y}^{(2)}$ as a response to prompt $x_i$ if for any model, we can make it more preferred by switching its response from $\tilde{y}^{(2)}$ to $\tilde{y}^{(1)}$. It is straightforward to verify that $\succsim_x$ from such a definition is transitive, continuous and satisfies independence since $\succsim$ does. Furthermore, independence of $\succsim$ implies $\succsim_x$ is complete. This can be seen via the utility representation in theorem \ref{kreps}. 

\subsection{Preference Axioms}
For a review, we cover the basic axioms on preferences that we assume throughout the paper. 
\begin{enumerate}
    \item Complete: for any $x, y$ in the mixture space it must be $x \succsim y$ or $y \succsim x$ or both. 
    \item Transitive: $x \succsim y \succsim z \implies x \succsim z$
    \item Continuous: For any $x \succ y \succ z$, there exists $a,b \in (0,1)$ such that $ax + (1-a)z \sim y \sim bx + (1 - b)z$
    \item Independence: $x \succ y \implies \alpha x + (1 - \alpha) z \succ \alpha y + (1- \alpha) z$
    for all $\alpha \in (0, 1]$.
\end{enumerate}

\newpage 
\clearpage
\section{Experiment Details}\label{app:experiment_details}
\subsection{Willingness-to-Pay Details}
In an ideal WTP setting, elicitations are incentivized by in some scenarios taking money from participants and allocating them the good. That is unfortunately difficult in a language model setting, but thankfully a frequent finding from experimental economics is that incentivized and unincentivized willingness-to-pay elicitations are not majorly dissimilar \citep{murphy2005meta}. Therefore we use stated willingness-to-pay in our experiment. 

To elicit these values, labelers are provided with labeling software that looks like the following 
\begin{figure}[h]
    \centering
    \includegraphics[width=0.95\columnwidth]{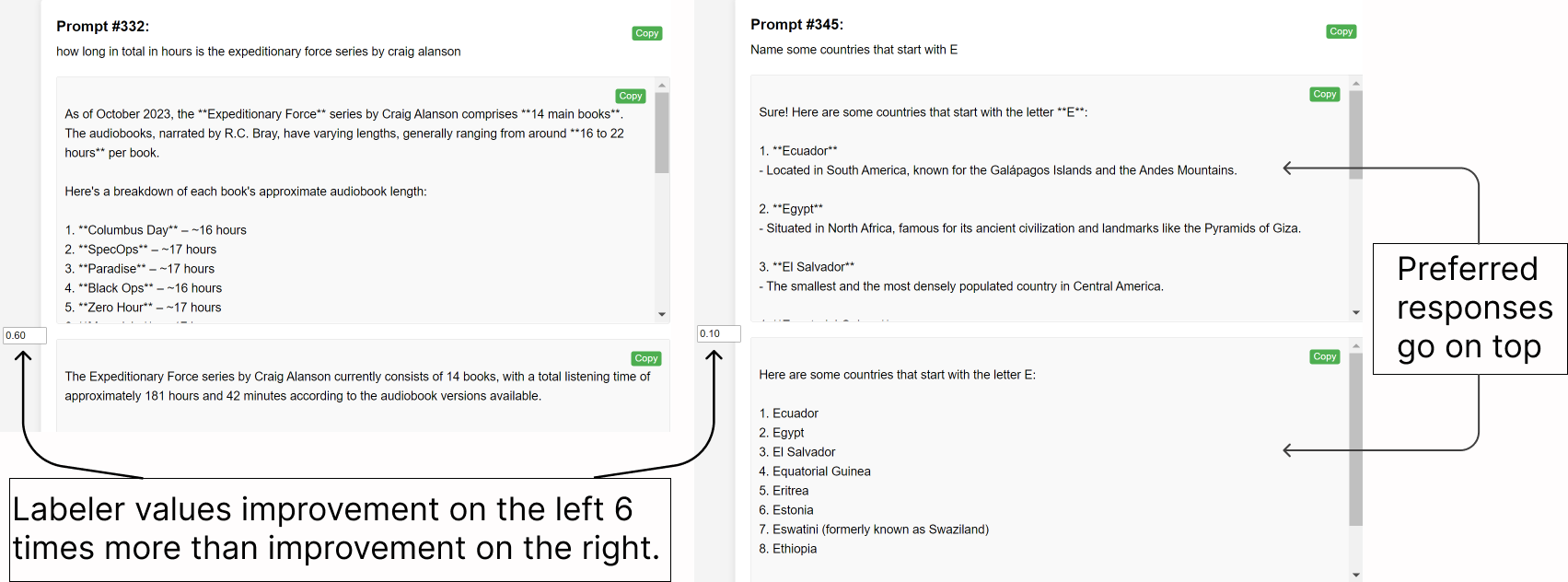}
    \caption{Labeling Software}
    \label{fig:labeling_software}
\end{figure}

Instructions are provided to labelers to rank their preferred responses on top and to use the numerical boxes to indicates their willingness-to-pay to move from the worse response to the better response. They are told that the WTP denotes the maximum you would be happy with spending to get the preferred response. If the amount were any higher, you would prefer to keep the money and get the worse response. Although conceptually the numeraire we focus on is money, we also tried various options such as normalizing the value of a reference improvement to $1$ and then eliciting all other values relative to that reference value. We include both of these methods in the final dataset to maximize sample size. 

Compensation was provided by the hour - at an average of roughly $\$15$ an hour, with bonuses for accuracy as measured by congruence.   

Ultimately, labels are provided by $5$ distinct labelers we hired - a mix of undergraduate and labelers on Upwork with LLM labeling experience. For the data we use in training, each sample corresponds to one labeler and there is no overlap between labelers. Separately, we designed a small hold-out set to check congruence between labelers and found a congruence of approximately $67\%$. 

Finally, for experiment one, the experiment was piloted with various options on how to set up the cross-labeling, and the one presented in the paper was the pilot option that was expanded after seeing promising results. 

\subsection{Labeling Improvements Prompt}
For Table 2, we use Claude $3.7$ Sonnet to label how much of a quality difference there is in responses and if the improvement in response quality is due to substantive or stylistic issues. The following is the prompt we gave Claude.

\begin{quote}
You are an expert AI assistant capable of evaluating the quality of text responses.
You will be given a prompt, a 'Chosen' response, and a 'Rejected' response.
Your task is to assess the *difference* in quality between the Chosen and Rejected responses.

First, rate the *importance* of this quality difference on a scale from 1 to 10, where:
1: The quality difference is negligible or very minor.
5: The quality difference is moderate and noticeable.
10: The quality difference is very significant and crucial.

Second, identify the primary nature of the improvement the 'Chosen' response has over the 'Rejected' response. Choose one:
- Substance: The chosen response is better due to significant improvements in factual accuracy, completeness, reasoning, instruction following, or overall helpfulness of the core content.
- Style: The chosen response is better primarily due to improvements in clarity, conciseness, tone, formatting, grammar, or readability, while the core substance might be similar or only marginally better.
- Other: Neither Substance nor Style predominantly describes the improvement, or the improvement is a mix, or it's hard to categorize.

Third, provide a brief explanation for your score and type.

Provide your response formatted *exactly* as follows, with each item on a new line:
Score: 
Improvement Type: 
Explanation: 
\end{quote}

\subsection{Details on Individual-Sample DPO}
In section 5, we ran a simple experiment where we tracked DPO loss at the individual sample level. Figure \ref{fig:mean_dpo_loss} shows the overall DPO loss follows a standard training path, so our results are not driven by failure to learn anything. 

\begin{figure}[h]
    \centering
    \includegraphics[width=.5\columnwidth]{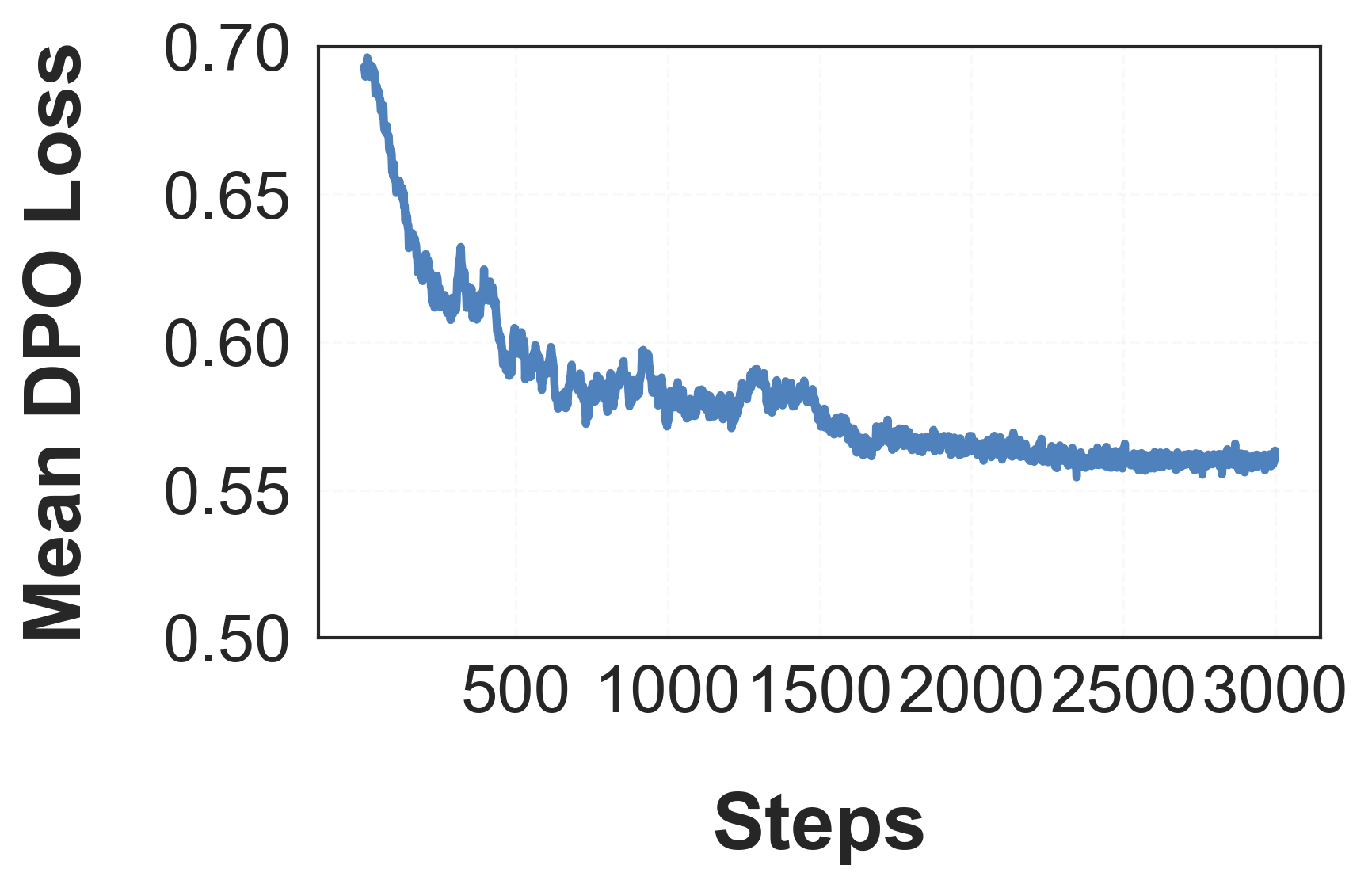}
    \caption{Mean DPO loss across all tracked 100 samples. }
    \label{fig:mean_dpo_loss}
\end{figure}

\subsection{Training Details}\label{appendix:training}

\textbf{DPO and CDPO Fine-tuning and Inference.}
For both DPO and CDPO, our base model is a Mistral-7B base model that has been full-parameter supervised fine-tuned on the UltraChat dataset. We then LoRA-fine-tune this model for one epoch. We use batch size 4, LoRA rank $r_{LoRA}=64$, regularization $\alpha_{LoRA} = 64$, and dropout $p_{LoRA} = 0.05$. We use learning rate $1e-5$, 150 warmup steps, and max conversation length of 1024 tokens. This replicates the Zephyr training recipe, except that we use $\beta = 0.1$ instead of $\beta = 0.01$, which we found improved performance. Finally, we normalize the WTP data so that the standard deviation matches the standard deviation from Bradley-Terry models, so that the regularization is relative to the same value for both values. 

We run this training recipe on a single A$100$ with $80$GB of VRAM. Each training run takes roughly $5$ hours. 

\textbf{Reward Model Training and Inference.}
To represent $U^{GT}$ for data labeling and evaluation, we use a reward model fine-tuned on \textsc{CardinalPrefs}. As the base reward model, we LoRA-fine-tune Skywork/Skywork-Reward-Llama-3.1-8B (a top-rated reward model on RewardBench) for one epoch on our dataset.

To compute the mean reward, we sample 1000 prompts from a held-out validation split of Ultrafeedback-200k that the language models were not trained on. We use this to calculate the average reward achieved by each model's completions.

\subsection{Evaluations}
For our accuracy plot, we use the implicit reward model defined by DPO and CDPO 
\begin{align}
    r(x_i,y) &= \beta \log \frac{\pi(y|x_i)}{\pi_{ref}(y|x)}
\end{align}
We use accuracy to mean the implicit reward model ranks $y_2$ and $y_1$ the same way $\succsim_{x_i}$ does. 

\section{Additional Empirical Results}
\subsection{Ablations}
We reproduce the first three rows of Table 2 while varying the regularization parameter $\beta = (0.01, 0.1, 0.25)$. For each choice of $\beta$, we implement the same training procedure described in the above section. These results show robustness across $\beta$ - CDPO consistently outperforms DPO, especially on importance-weighted tasks.  

\begin{table}[h!]
\centering
\caption{Ablation Study on Regularization Parameter $\beta$}
\label{tab:beta_ablation_designed}
\begin{tabular}{@{}lcc@{}} 
\toprule
\textbf{Metric} & \textbf{DPO (\%)} & \textbf{CDPO (\%)} \\ 
\midrule
\multicolumn{3}{@{}l}{\textit{ $\beta = 0.1$}} \\ 
\quad Outputs Preferred Response & 51.1\% & \textbf{51.6\%} \\
\quad Outputs Preferred Response (WTP Weighted) & 52.8\% & \textbf{56.7\%} \\
\quad Outputs Preferred Response (Importance Weighted) & 46.0\% & \textbf{49.0\%} \\
\midrule
\multicolumn{3}{@{}l}{\textit{$\beta = 0.01$}} \\ 
\quad Outputs Preferred Response & 50.0\% & \textbf{65.0\%} \\
\quad Outputs Preferred Response (WTP Weighted) & 47.2\% & \textbf{75.6\%} \\
\quad Outputs Preferred Response (Importance Weighted) & 47.7 & \textbf{65.6} \\
\midrule
\multicolumn{3}{@{}l}{\textit{$\beta = 0.25$}} \\ 
\quad Outputs Preferred Response & \textbf{34.0\%} & \textbf{34.0\%} \\ 
\quad Outputs Preferred Response (WTP Weighted) & 22.8\% & \textbf{28.2\%} \\
\quad Outputs Preferred Response (Importance Weighted) & 24.78\% & \textbf{26.9\%} \\
\bottomrule
\end{tabular}
\end{table}

\subsection{Objective Benchmark Results}
We compare DPO and CDPO on objective benchmarks to supplement our preference-based metrics. On the objective benchmarks such as TruthfulQA, AGIEval and MMLU, the comparisons are closer than on preference-based metrics, with perhaps a slight advantage for CDPO. This result indicates most of the benefit of CDPO is preference-based rather than learning new capabilities or facts, which is expected for a post-training, preference-based method. 

\begin{table}[h!]
    \centering
    \caption{CDPO exceeds DPO by prioritizing important cases.}
    \begin{tabular}{lcc}
        \toprule
        TruthfulQA & 33.3\% & \textbf{33.7\%} \\
        MMLU & \textbf{56.7}\% & 56.4\% \\
        AGIEval & 30.5\% & \textbf{30.6\%} \\
        \bottomrule
    \end{tabular}
    \label{tab:model_comparison}
\end{table}

\subsection{Table 2 in Graph Form}
To parse how CDPO prioritizes vs. DPO, we turn part of Table 2 into Figure \ref{fig:final}. Here, we can see that on substantive issues CDPO exceeds DPO but on style DPO greatly exceeds CDPO. This supports the hypothesis that DPO focuses on unimportant improvements while CDPO puts more optimization pressure on important, core improvements.

\newpage 
\clearpage

\begin{figure}[h!]
    \centering
    \includegraphics[width=\linewidth]{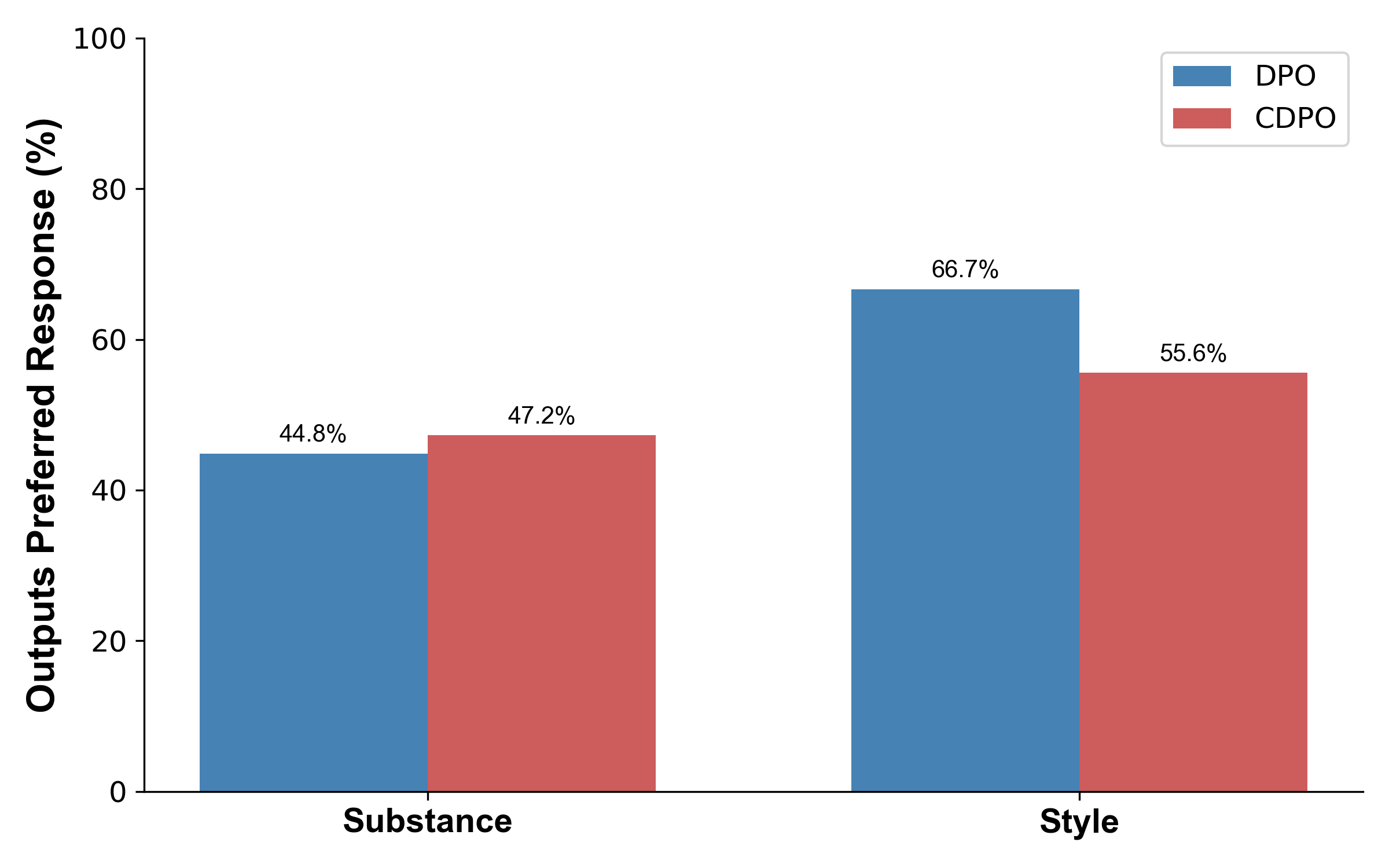}
    \caption{CDPO exceeds DPO on substantive improvements, while DPO exceeds on style improvements.  }
    \label{fig:final}
\end{figure}

\newpage
\section*{NeurIPS Paper Checklist}


\begin{enumerate}

\item {\bf Claims}
    \item[] Question: Do the main claims made in the abstract and introduction accurately reflect the paper's contributions and scope?
    \item[] Answer: \answerYes{} 
    \item[] Justification: The abstract and intro show the central contribution is about using richer feedback in RLHF/DPO to get better performance. 

\item {\bf Limitations}
    \item[] Question: Does the paper discuss the limitations of the work performed by the authors?
    \item[] Answer: \answerYes{} 
    \item[] Justification: This is discussed in the conclusion and in the experimental section where we discuss the weakness of each experiment. We also discuss the cost/time increase of using cardinal labels.  

\item {\bf Theory assumptions and proofs}
    \item[] Question: For each theoretical result, does the paper provide the full set of assumptions and a complete (and correct) proof?
    \item[] Answer: \answerYes{}
    \item[] Justification: All theorem statements are complete and the proofs are in the appendix. The full set of assumptions are always stated in the theorem or stated as assumed throughout the paper. 

    \item {\bf Experimental result reproducibility}
    \item[] Question: Does the paper fully disclose all the information needed to reproduce the main experimental results of the paper to the extent that it affects the main claims and/or conclusions of the paper (regardless of whether the code and data are provided or not)?
    \item[] Answer: \answerYes{}{} 
    \item[] Justification: Appendix describes our data collection, and training procedure exactly. 

\item {\bf Open access to data and code}
    \item[] Question: Does the paper provide open access to the data and code, with sufficient instructions to faithfully reproduce the main experimental results, as described in supplemental material?
    \item[] Answer: \answerYes{} 
    \item[] Justification: We have open sourced our core dataset. It is available at \url{https://huggingface.co/datasets/cardinal-prefs/CardinalPrefs}. There are sufficient instructions to replicate all experimental results, although we have not open sourced our code yet. 

\item {\bf Experimental setting/details}
    \item[] Question: Does the paper specify all the training and test details (e.g., data splits, hyperparameters, how they were chosen, type of optimizer, etc.) necessary to understand the results?
    \item[] Answer: \answerYes{}{} 
    \item[] Justification: Yes, training details are in the appendix. 

\item {\bf Experiment statistical significance}
    \item[] Question: Does the paper report error bars suitably and correctly defined or other appropriate information about the statistical significance of the experiments?
    \item[] Answer: \answerYes{}{} 
    \item[] Justification: Experiment one shows statistical significance, the result of the experiments are graphical or benchmarks scores which are not amenable to statistical significance testing. 

\item {\bf Experiments compute resources}
    \item[] Question: For each experiment, does the paper provide sufficient information on the computer resources (type of compute workers, memory, time of execution) needed to reproduce the experiments?
    \item[] Answer: \answerYes{} 
    \item[] Justification: Yes, we describe using one A100 throughout the paper. 
    
\item {\bf Code of ethics}
    \item[] Question: Does the research conducted in the paper conform, in every respect, with the NeurIPS Code of Ethics \url{https://neurips.cc/public/EthicsGuidelines}?
    \item[] Answer: \answerYes{} 
    \item[] Justification: To our knowledge it does.

\item {\bf Broader impacts}
    \item[] Question: Does the paper discuss both potential positive societal impacts and negative societal impacts of the work performed?
    \item[] Answer: \answerYes{}{} 
    \item[] Justification: See the impact statement at the end. 
    
\item {\bf Safeguards}
    \item[] Question: Does the paper describe safeguards that have been put in place for responsible release of data or models that have a high risk for misuse (e.g., pretrained language models, image generators, or scraped datasets)?
    \item[] Answer: \answerNo{} 
    \item[] Justification: Our models are not frontier or sensitive. For our data release, data is anonymized to protect the identity of labelers. 

\item {\bf Licenses for existing assets}
    \item[] Question: Are the creators or original owners of assets (e.g., code, data, models), used in the paper, properly credited and are the license and terms of use explicitly mentioned and properly respected?
    \item[] Answer: \answerYes{} 
    \item[] Justification: Whenever we use data, we properly credit the source. 

\item {\bf New assets}
    \item[] Question: Are new assets introduced in the paper well documented and is the documentation provided alongside the assets?
    \item[] Answer: \answerYes{}
    \item Justification: When introducing novel data, we note that. 

\item {\bf Crowdsourcing and research with human subjects}
    \item[] Question: For crowdsourcing experiments and research with human subjects, does the paper include the full text of instructions given to participants and screenshots, if applicable, as well as details about compensation (if any)? 
    \item[] Answer: \answerYes{} 
    \item[] Justification: Yes, see the appendix for the presented GUI, compensation details and labeling instructions. 

\item {\bf Institutional review board (IRB) approvals or equivalent for research with human subjects}
    \item[] Question: Does the paper describe potential risks incurred by study participants, whether such risks were disclosed to the subjects, and whether Institutional Review Board (IRB) approvals (or an equivalent approval/review based on the requirements of your country or institution) were obtained?
    \item[] Answer: \answerYes{} 
    \item[] Justification: We have IRB approval. As study participants were providing non-sensitive data, it is low to no risk. 

\item {\bf Declaration of LLM usage}
    \item[] Question: Does the paper describe the usage of LLMs if it is an important, original, or non-standard component of the core methods in this research? Note that if the LLM is used only for writing, editing, or formatting purposes and does not impact the core methodology, scientific rigorousness, or originality of the research, declaration is not required.
    \item[] Answer: \answerYes{} 
    \item[] Justification: We only use LLMs for formatting and code assistance except for when we use Claude to label data as described in the text. 
\end{enumerate}

\end{document}